%% file: main.tex
\documentclass{article}
\usepackage{ijcai24}

\usepackage{times}
\usepackage[utf8]{inputenc}
\usepackage[small]{caption}

\usepackage{soul}
\usepackage{amsmath}
\usepackage{amssymb}
\usepackage{amsfonts}
\usepackage{amsthm}
\usepackage{url}
\usepackage{caption}
\usepackage{xspace}
\usepackage{color}
\usepackage[dvipsnames]{xcolor}
\usepackage{multirow}
\usepackage{booktabs}
\usepackage{graphicx}
\usepackage[normalem]{ulem}
\usepackage{accents}
\usepackage{comment}
\usepackage{makecell}
\usepackage{enumitem}
\usepackage[hidelinks]{hyperref}
\usepackage[switch]{lineno}

\usepackage[noend,ruled,linesnumbered]{algorithm2e}
\usepackage{placeins}
\SetKwInOut{Input}{input}

\input{macros}
\input{environments}
\input{tabella}

\newif\ifjournal
\journalfalse

\newif\ifcameraready
\camerareadytrue

\newif\ifdraft
\drafttrue

\ifdraft
\newcommand{\nb}[1]{\textcolor{red}{\bf!}%
	\marginpar[\parbox{14.5mm}{\raggedleft\scriptsize\textcolor{red}{#1}}]%
	{\parbox{14.5mm}{\raggedright\scriptsize\textcolor{red}{#1}}}}
\newcommand{\warn}[1]{\textcolor{red}{#1}}
\else
\newcommand{\nb}  [1]{}
\newcommand{\warn}[1]{}
\fi

\ifjournal\else
\fi

\title{\ifjournal
Controlled Query Evaluation under Epistemic Policies
\else
Controlled Query Evaluation through Epistemic Dependencies
\fi}
\author{Gianluca Cima$^1$\and
        Domenico Lembo$^1$\and
        Lorenzo Marconi$^1$\and
        \\Riccardo Rosati$^1$\And
        Domenico Fabio Savo$^2$\\
\affiliations
$^1$Sapienza University of Rome\\
$^2$University of Bergamo\\
\emails
\{lastname\}@diag.uniroma1.it,
domenicofabio.savo@unibg.it
}
\begin{document}

\maketitle

\begin{abstract}
    In this paper, we propose the use of epistemic dependencies to express data protection policies in Controlled Query Evaluation (CQE), which is a form of confidentiality-preserving query answering over ontologies and databases.
    The resulting policy language goes significantly beyond those proposed in the literature on CQE so far, allowing for very rich and practically interesting forms of data protection rules. We show the expressive abilities of our framework and study the data complexity of CQE for (unions of) conjunctive queries when ontologies are specified in the Description Logic $\dlliter$. Interestingly, while we show that the problem is in general intractable, we prove tractability for the case of acyclic epistemic dependencies by providing a suitable query rewriting algorithm. The latter result paves the way towards the implementation and practical application of this new approach to CQE.
\end{abstract}

\sloppy

\input{1-introduction}
\input{2-preliminaries}
\input{3-framework}

\input{4-indistinguishability}

\input{5.1-results-cyclic}

\input{5.2-results-acyclic}
\input{6-conclusions}


\ifcameraready
\section*{Acknowledgements}
This work was partially supported by: projects FAIR (PE0000013) and SERICS (PE00000014) under the MUR National Recovery and Resilience Plan funded by the EU - NextGenerationEU; GLACIATION project funded by the EU (N. 101070141); ANTHEM (AdvaNced Technologies for Human-centrEd Medicine) project (CUP B53C22006700001) funded by the National Plan for NRRP Complementary Investments; the MUR PRIN 2022LA8XBH project Polar (POLicy specificAtion and enfoRcement for privacy-enhanced data management); and by the EU under the H2020-EU.2.1.1 project TAILOR (grant id.~952215).
\fi

\bibliographystyle{named}
\bibliography{bibliography/medium-string,bibliography/krdb,bibliography/w3c,bibliography/local-bib}

\ifjournal\else
    \input{appendix}
\fi

\end{document}

%% file: macros.tex
\newcommand{\A}{\mathcal{A}}
\newcommand{\C}{\mathcal{C}}

\newcommand{\E}{\mathcal{E}}

\newcommand{\I}{\mathcal{I}}

\renewcommand{\O}{\mathcal{O}}
\renewcommand{\P}{\mathcal{P}}
\newcommand{\R}{\mathcal{R}}
\newcommand{\Q}{\mathcal{Q}}
\renewcommand{\S}{\mathcal{S}}
\newcommand{\T}{\mathcal{T}}

\newcommand{\conceptSet}{\Gamma_C}
\newcommand{\roleSet}{\Gamma_R}
\newcommand{\individualSet}{\Gamma_I}
\newcommand{\variableSet}{\Gamma_V}
\newcommand{\labeledNullSet}{\Gamma_N}

\newcommand{\tup}[1]{\langle #1 \rangle}

\newcommand{\body}{\textit{body}}
\newcommand{\head}{\textit{head}}

\newcommand{\aczero}{\mathrm{AC}^0}

\newcommand{\ra}{\rightarrow}

\newcommand{\chase}{\mathit{Chase}}

\newcommand{\TGD}{\mathit{TGD}}

\newcommand{\modelseql}{\models_\mathsf{EQL}}

\newcommand{\bcqclosure}{\textit{\BCQ}\text{-}Cons}
\newcommand{\bcqkclosure}{\textit{\BCQ}_k\text{-}Cons}
\newcommand{\closure}{\textit{Cons}}

\newcommand{\optcqcens}{\textit{OptCQCens}}

\newcommand{\dlliter}{\text{DL-Lite}_\R}

\newcommand{\cqlength}{\textit{Len}}
\newcommand{\maxlength}{\textit{MaxLenCQ}}
\newcommand{\maxlengthatoms}{\textit{MaxLenCQ}}

\newcommand{\BCQ}{\textbf{BCQ}}
\newcommand{\BUCQ}{\textbf{BUCQ}}

\newcommand{\BCQk}{\textbf{BCQ}_k}

\newcommand{\BCQh}{\textbf{BCQ}_h}
\newcommand{\BCQl}{\textbf{BCQ}_{\ell}}

\newcommand{\BUCQk}{\textbf{BUCQ}_k}
\newcommand{\eval}{\textit{eval}}
\newcommand{\perfectref}{\textit{PerfectRef}}

\newcommand{\kcons}{\textit{PolicyCons}}
\newcommand{\phikcons}{\phi_\mathit{pc}}
\newcommand{\intcens}{\C_{\mathit{int}}(\E)}
\newcommand{\clash}{\textit{Clash}}

\newcommand{\cqe}{\text{SC}}
\newcommand{\icqe}{\text{IC}}
\newcommand{\censentailed}{\textit{\cqe-Entailed}}
\newcommand{\icensentailed}{\textit{\icqe-Entailed}}
\newcommand{\modelscqe}{\models_{\cqe}}
\newcommand{\modelsicqe}{\models_{\icqe}}

\newcommand{\coNP}{coNP\xspace}

\newcommand{\decprob}{\textsf{REC}}
\newcommand{\acdecprob}{\textsf{AREC}}

\newcommand{\qedex}{\hfill\ensuremath{\vrule height 4pt width 4pt depth 0pt}}

\newcommand{\ISA}{\sqsubseteq}

\newcommand{\lang}{\mathcal{L}}
\newcommand{\LT}{\lang_\T}
\newcommand{\LQ}{\lang_q}

\newcommand{\constpred}{\textit{Ind}}
\newcommand{\ucqrewrite}{\textit{UCQRewrite}}
\newcommand{\andcqs}{\textit{And}}

\newcommand{\comA}{C$_A$\xspace}
\newcommand{\comB}{C$_B$\xspace}

\newcommand{\profiledActivity}{\mathsf{profiledActivity}}
\newcommand{\citOf}{\mathsf{citOf}}
\newcommand{\SR}{\mathsf{SR}}
\newcommand{\consent}{\mathsf{Consent}}
\newcommand{\name}{\mathsf{name}}
\newcommand{\dateBirth}{\mathsf{dateB}}

\newcommand{\pA}{\mathsf{p_1}}
\newcommand{\pB}{\mathsf{p_2}}
\newcommand{\actA}{\mathsf{act_1}}
\newcommand{\actB}{\mathsf{act_2}}
\newcommand{\ann}{\mathsf{ann}}

\newcommand{\dA}{\mathsf{date_1}}

\newcommand{\patient}{\mathsf{Patient}}
\newcommand{\admitted}{\mathsf{admitted}}

%% file: environments.tex
\newtheorem{example}{Example}
\newtheorem{theorem}{Theorem}
\newtheorem{definition}{Definition}
\newtheorem{proposition}{Proposition}
\newtheorem{lemma}{Lemma}

\newenvironment{proofsk}{\noindent \textsl{Proof (sketch).\ }}{\qed}

%% file: tabella.tex
\newcommand{\tabella}{
\begin{table}
\centering
\setlength\tabcolsep{3pt} 
\begin{tabular}{|c|c|c|c|c|c|}
\hline
\multirow[c]{3}{*}{Policy} & Query & \multicolumn{2}{c|}{\mbox{} Confidentiality} & \multicolumn{2}{c|}{Data} \\
 & language & \multicolumn{2}{c|}{\mbox{} preservation} & \multicolumn{2}{c|}{complexity} \\
\cline{3-6}
& & \cqe-ent. & \icqe-ent. & \cqe-ent. & \icqe-ent. \\
\hline
Acyclic & BCQ & \multicolumn{2}{c|}{yes} & \multicolumn{2}{c|}{in $\aczero$} \\
\hline
Arbitrary & BCQ & \multicolumn{2}{c|}{yes} & \multicolumn{2}{c|}{\coNP-c} \\
\hline
Acyclic & BUCQ & no & yes & in $\aczero$ & in $\aczero$ \\
\hline
Arbitrary & BUCQ & no & yes & \coNP-c & \coNP-c \\
\hline
\end{tabular}
\caption{Results on data complexity and confidentiality preservation. \coNP-c stands for \coNP-complete.}
\label{tab:results}
\end{table}
}

\newcommand{\tabellaref}{
\begin{tableref}
\centering
\setlength\tabcolsep{3pt} 
\begin{tabular}{|c|c|c|c|c|c|}
\hline
\multirow[c]{3}{*}{Policy} & Query & \multicolumn{2}{c|}{\mbox{} Confidentiality} & \multicolumn{2}{c|}{Data} \\
 & language & \multicolumn{2}{c|}{\mbox{} preservation} & \multicolumn{2}{c|}{complexity} \\
\cline{3-6}
& & \cqe-ent. & \icqe-ent. & \cqe-ent. & \icqe-ent. \\
\hline
Acyclic & BCQ & \multicolumn{2}{c|}{yes} & \multicolumn{2}{c|}{in $\aczero$} \\
\hline
Arbitrary & BCQ & \multicolumn{2}{c|}{yes} & \multicolumn{2}{c|}{\coNP-c} \\
\hline
Acyclic & BUCQ & no & yes & in $\aczero$ [\ref{thm:cqe-entailment-full}] & in $\aczero$ [\ref{thm:icqe-entailment-full}] \\
\hline
Arbitrary & BUCQ & no & yes & \coNP-c [\ref{thm:cqe-entailment-bcq-ub}] & \coNP-c [\ref{thm:icqe-entailment-bcq-ub}] \\
\hline
\end{tabular}
\caption{Results on data complexity and confidentiality preservation. \coNP-c stands for \coNP-complete.}
\label{tab:results}
\end{table}
}

%% file: 1-introduction.tex
\section{Introduction}
\label{sec:introduction}

Controlled Query Evaluation (CQE) 
is a confidentiality-preserving query answering approach that protects sensitive information 
by filtering query answers in such a way that a user 
cannot deduce information declared confidential by a data protection policy~\cite{Bisk00,BiBo04,BoSa13,CKKZ13}. 

In CQE, one crucial aspect concerns the expressiveness of the policy language, which determines the form of the logic formulas definable in the policy, consequently influencing a designer's capacity to declare which pieces of information must not be disclosed. 
Previous literature has mainly considered only policies consisting of sentences, i.e.\ closed formulas, as in~\cite{SiJR83,BiBo04,BoSa13,LeRS19}. Through this approach, it is only possible to impose that the truth value of a sentence in the policy cannot be inferred from the system by asking queries. For example, \cite{LeRS19} study CQE when policy statements take the form $\forall\vec{x}(q(\vec{x}) \ra \bot)$, referred to as denial. 
Enforcing one such denial over the data means refraining from disclosing the inference of the Boolean conjunctive query (BCQ) $\exists \vec{x}~q(\vec{x})$ by the system, even if the inference has occurred. For instance, the rule
\begin{center}
\begin{tabbing}
\hspace{1cm}$\delta_1 = \forall x,y (\patient(x) \wedge \admitted(x,y) \rightarrow \bot)$
\end{tabbing}
\end{center}
says that it is confidential whether there exists a patient admitted in a hospital department.
However, if the aim is to hide admitted patients, the above dependency is 
imposing an excessively stringent level of protection, 
since it is denying the presence of patients 
in the hospital to protect their identity, 
which is implausible in practice.
A more effective approach would be instead to require that the system must not answer the \emph{open} query $q(x): \exists y (\patient(x) \wedge \admitted(x,y))$. Intuitively, specifying a rule of this kind means imposing that the set of patients that the system \emph{knows} to be admitted to the hospital has not to be disclosed.

To properly capture this behavior, we propose to use an epistemic operator $K$ in policy formulas, in the spirit of~\cite{CDLLR07,CoLe20}. This allows us to formalize the epistemic state of the user, that is, what the system can tell to the user without disclosing sensitive information. In our example, the policy rule is as follows:
%
\begin{center}
\begin{tabbing}
\hspace{0.4cm}$\delta_2 = \forall x \big( K\, \exists y (\patient(x) \wedge \admitted(x,y)) \rightarrow K \bot \big)$
\end{tabbing}
\end{center}
Rule $\delta_2$ imposes that, in the epistemic state of the user, the set of admitted patients must be empty (but this does not exclude that the user knows that some patients have been admitted).

In fact, our proposal enables us to accomplish more than just that. In a more advanced scenario, concealing the relationship between a patient and a hospital should apply only if 
the patient has not signed a consensus form. This can be expressed by the following formula:
\begin{center}
\begin{tabbing}
\hspace{0.85cm} $\delta_3 = \forall x \big(K\,\exists y (\patient(x) \wedge\admitted(x,y)) \ra$\\
\hspace{5.7cm}$ K \consent(x) \big)$
\end{tabbing}
\end{center}
Intuitively, the formula is saying that if a user knows that a patient has been admitted, then she must know that the patient has signed a consensus form. Thus, if a patient did not sign a consensus form, the system cannot disclose that this patient has been admitted.
In general, this kind of policy is well-suited for encoding the principle of \emph{privacy by default}, which is a desirable property in data protection (as expressly outlined 
in Article~25 of GDPR~\cite{GDPR-art25}).

We remark once again that policy rules of the form just described have not been previously considered in the literature. It is however worth noticing that in~\cite{CKKZ15} CQE is studied for a policy consisting of \emph{one single open} CQ. In this latter framework, a rule analogous to $\delta_2$ of our example can be in fact specified, but richer policies, as those requiring more denials and/or rules as $\delta_3$ above, are non-expressible.

Formulas as $\delta_2$ and $\delta_3$ are called \emph{epistemic dependencies} (EDs). EDs have been originally introduced in~\cite{CoLe20} to express integrity constraints in ontology-based data management and are indeed a special case of \emph{domain-independent EQL-Lite(CQ)} sentences~\cite{CDLLR07b}. 
In the present paper, we use EDs as policy rules that must be satisfied to preserve data confidentiality over Description Logic (DL) ontologies, similarly as integrity constraints must be satisfied to guarantee data consistency. 
However, our aim is totally different from that of~\cite{CoLe20}, where the focus is on consistency checking.


After defining the policy language, to completely characterize our CQE framework, we need to specify its semantics.
%
This issue is addressed in CQE through \emph{censors}. 
In this paper, we use \emph{CQ-censors} introduced in~\cite{LeRS19}. 
In a nutshell, given an ontology $\O$, an \emph{optimal CQ-censor} is a maximal subset $\C$ of the set of all BCQs inferred by $\O$, such that $\C$ coupled with the intensional component of $\O$ (i.e.\ its TBox $\T$) satisfies all rules in the policy (i.e.\ no secrets are disclosed through standard query answering over $\T \cup \C$).
As in~\cite{LeRS19}, we then define CQE as the problem of computing the query answers that are in the intersection of the answer sets returned by all optimal censors. 
We call this problem \emph{SC-entailment} (Skeptical entailment under CQ-censors).
This form of CQE does not suffer of the problem of having to arbitrarily select an optimal censor among several incomparable ones, as done, e.g., in~\cite{BoSa13,CKKZ15}. 
We point out that CQ-censors are particularly interesting from a practical perspective, since, 
for ontologies specified $\dlliter$~\cite{CDLLR07}, a DL suited for modeling data intensive ontologies, and policies expressed as denials, SC-entailment of BCQs is tractable in data complexity~\cite{LeRS19}. One of the aims of the present paper is then to study data complexity of SC-entailment of BCQs under epistemic policies for $\dlliter$ ontologies, with the ultimate goal of identifying conditions ensuring tractability in this case. 

%

Besides the computational complexity study, we also carry out an analysis of the robustness of SC-entailment with respect to confidentiality-preservation. In~\cite{BiWe08,BoSa13,Bona22}, it is shown that censoring mechanisms based on an indistinguishability criterion are indeed more secure than others. 
In abstract terms, according to such a criterion, confidentiality is guaranteed only if query answers returned over a data instance with sensitive information coincide with those returned 
over a data instance without secrets (which is thus indistinguishable from the other instance).
In this paper, we investigate whether the entailment we consider enjoys indistinguishability. 



Specifically, our main results are for ontologies in $\dlliter$ and policies that are sets of EDs. In more detail:

\begin{itemize}
\setlength\itemsep{0em}
\item As for indistinguishability, we show that SC-entailment of BCQs preserves confidentiality as defined in~\cite{BiWe08}, but that this result does not carry over to unions of BCQs (BUCQs). We however prove that the property holds even for BUCQs in the case of IC-entailment, a sound approximation of SC-Entailment that considers a single censor given by the intersection of all optimal CQ-censors (Section~\ref{sec:indistinguishability});
\item We show that SC-entailment and IC-entailment of BCQs and BUCQs are coNP-complete in data complexity (Section~\ref{sec:arbitrary});
\item For \emph{acyclic EDs}, we exhibit a rewriting algorithm that allows us to prove that SC-entailment and IC-entailment of BCQs and BUCQs are first-order rewritable, and thus in AC$^0$ in data complexity (Section~\ref{sec:acyclic}).
\end{itemize}

\tabella

All our results are summarized in Table~\ref{tab:results}.
Complete proofs are given in the appendix.

%% file: 2-preliminaries.tex
\section{Preliminaries} 
\label{sec:preliminaries}

We employ standard notions of function-free first-order logic (FO), and consider FO formulas using only unary and binary predicates called \emph{concepts} and \emph{roles}, respectively, as in Description Logics, which are fragments of FO suited for conceptual modeling~\cite{BKNP20}. 
We assume the existence of pairwise-disjoint countably-infinite sets $\conceptSet$, $\roleSet$, $\individualSet$, $\labeledNullSet$, and $\variableSet$ containing \emph{atomic concepts}, \emph{atomic roles}, \emph{constants} (also known as \emph{individuals}), \emph{labeled nulls}, and \emph{variables}, respectively.
An FO formula $\phi$ is sometimes denoted as $\phi(\vec{x})$, where $\vec{x}$ is the sequence of the free variables occurring in $\phi$. We also use the term \emph{query} as a synonym of FO formula and \emph{Boolean query} as a synonym of closed FO formula (also called sentence).
An FO theory $\Phi$ is a set of FO sentences. The semantics of $\Phi$ is given in terms of FO interpretations over $\conceptSet \cup \roleSet \cup \individualSet$.
W.l.o.g., we consider interpretations sharing the same infinite countable domain $\Delta=\individualSet$, so that every element in $\individualSet$ is interpreted by itself. In other terms, we use \emph{standard names}, as often customary when one deals with epistemic operators~\cite{CDLLR07b}.
We write $\eval(\phi,\I)$ to indicate the evaluation of an FO sentence $\phi$ over an FO interpretation $\I$.
A \emph{model} of an FO theory $\Phi$ is an FO interpretation satisfying all sentences in $\Phi$. We say that $\Phi$ \emph{entails} an FO sentence $\phi$, 
denoted by $\Phi \models \phi$, if $\eval(\phi,\I)$ is true in every model $\I$ of $\Phi$.

A \emph{Description Logic (DL) ontology} $\O = \T \cup \A$ consists of a TBox $\T$ and an ABox $\A$, representing intensional and extensional knowledge, respectively.
In this paper, an ABox is a finite set of atoms using predicate symbols from $\conceptSet\cup\roleSet$ and terms from $\individualSet\cup\labeledNullSet$ (ABoxes of this form are also called quantified ABoxes~\cite{BKNP20}).
A model of an ontology $\T\cup\A$ is any model of the FO theory $\T \cup \{\exists\vec{x}\,\phi_\A(\vec{x})\}$, where $\vec{x}$ is a sequence of variables from $\variableSet$ and $\phi_\A(\vec{x})$ is the conjunction of all the atoms of $\A$ in which each labeled null 
is replaced with a distinct variable $x\in\vec{x}$.
With a little abuse of notation, we sometimes treat an ontology $\T\cup\A$ (resp.\ $\A$) as the FO theory $\T \cup \{\exists\vec{x}\,\phi_\A(\vec{x})\}$ (resp.\ $\{\exists\vec{x}\,\phi_\A(\vec{x})\}$).
For instance, this allows us to write $\T\cup\A \models \phi$ to intend $\T \cup \{\exists\vec{x}\,\phi_\A(\vec{x})\} \models \phi$, where $\phi$ is an FO sentence. 

Our complexity results are given for ontologies expressed in $\dlliter$, a member of the well-known DL-Lite family of DLs \cite{CDLLR07}. 
A $\dlliter$ TBox $\T$ is a finite set of axioms of the form $B \sqsubseteq B'$ and $R \sqsubseteq R'$ (called concept and role inclusions), or $B \sqsubseteq \neg B'$ and $R \sqsubseteq \neg R'$ (called concept and role disjointnesses), where $B$ and $B'$ (resp., $R$ and $R'$) are predicates of the form $A$, $\exists S$ or $\exists S^-$ (resp., $S$ or $S^-$) such that $A\in\conceptSet$, $S\in\roleSet$, $S^-$ is the inverse of $S$, and the \emph{unqualified existential restrictions} $\exists S$ and $\exists S^-$ represent the set of objects appearing as the first and second argument of $S$, respectively.

As usual when studying query answering over DL ontologies, we focus on the language of conjunctive queries (and variants thereof).
A \emph{conjunctive query} (\emph{CQ}) takes the form of an FO formula $\exists\vec{y}\phi(\vec{x},\vec{y})$, where $\vec{x}\cup\vec{y}\subseteq\variableSet$ and $\phi(\vec{x},\vec{y})$ is a finite, non-empty conjunction of atoms of the form $\alpha(\vec{t})$, where $\alpha\in\conceptSet\cup\roleSet$, and each term in $\vec{t}$ is either a constant in $\individualSet$ or a variable in $\vec{x}\cup\vec{y}$.
We also consider the special CQ $\bot$ and assume that $\eval(\bot,\I)$ is false for every FO interpretation $\I$.
A \emph{union of conjunctive queries} (\emph{UCQ}) is a disjunction $q_1(\vec{x})\lor\ldots\lor q_n(\vec{x})$ of CQs.
For convenience, sometimes we treat UCQs as sets of CQs.
Boolean CQs and UCQs, for short, are respectively indicated as \emph{BCQs} and \emph{BUCQs}.

Given a CQ $q$, we denote by $\cqlength(q)$ the number of atoms in $q$. Given a UCQ $q$, 
$\maxlength(q)=\max_{q'\in q}\cqlength(q')$.
%
We denote by $\BCQ$ (resp.~$\BUCQ$) the language of all the BCQs (resp.~BUCQs), and, for a positive integer $k$, by $\BCQk$ (resp.~$\BUCQk$) the language of BCQs (resp.~BUCQs) $q$ such that $\cqlength(q)\le k$ (resp., $\maxlengthatoms(q)\le k$).
Given a TBox $\T$, an ABox $\A$, and a Boolean query language $\LQ\subseteq\BCQ$, we denote by $\LQ$-$\closure(\T \cup \A)$ the set of Boolean queries $q \in \LQ$ that are logical consequences of $\T\cup\A$. Formally: $\LQ$-$\closure(\T \cup \A)=\{q \in \LQ \mid \T \cup \A \models q\}$
%


%
A \emph{ground substitution} for a sequence $\vec{x}=x_1,\ldots,x_k$ of variables is a sequence of constants $\vec{c}=c_1,\ldots,c_k$. Furthermore, if $\vec{x}$ are the free variables of a FO formula $\phi(\vec{x})$, we indicate as $\phi(\vec{c})$ the FO sentence obtained from $\phi(\vec{x})$ by replacing each $x_i$ with $c_i$, for $1\le i\le k$.

We recall that query answering of UCQs in $\dlliter$ is first-order rewritable, i.e.\ for every $\dlliter$ TBox $\T$ and UCQ $q(\vec{x})$, it is possible to compute 
an FO query $q_r(\vec{x})$ such that, for every ground substitution $\vec{c}$ of $\vec{x}$, $\T\cup\A\models q(\vec{c})$ iff $\A \models q_r(\vec{c})$.
To compute $q_r(\vec{x})$, we use the algorithm $\perfectref$ presented in~\cite{CDLLR07}, for which the following property holds.

%
\begin{proposition}
\label{pro:rewriting-dlliter}
    Let $\T$ be a $\dlliter$ TBox and let $q(\vec{x})$ be a UCQ. For every ABox $\A$ and every ground substitution $\vec{c}$ for $\vec{x}$, we have that $\T \cup \A \models q(\vec{c})$ if and only if 
    $\A\models q_r(\vec{c})$,
    where $q_r(\vec{x})=\perfectref(q(\vec{x}),\T)$.
\end{proposition}
We point out that, by construction, $q_r=\perfectref(q,\T)$ is a UCQ and that $\maxlengthatoms(q_r) = \maxlengthatoms(q)$.



%% file: 3-framework.tex
\section{Framework} 
\label{sec:framework}
In this section, we describe our CQE framework. 
We first give the notion of epistemic dependencies that we use in the policies, then present the notion of censor, and finally provide two notions of query entailment in our novel framework.   


\smallskip

\noindent \textbf{Epistemic dependencies.} The 
policy $\P$ of our framework is a finite set of \emph{epistemic dependencies}, each of which can be seen as a \emph{domain-independent EQL-Lite(CQ)}~\cite{CDLLR07b} sentence defined as follows.
\begin{definition}
\label{def:epistemic-dependency}
    An \emph{epistemic dependency (ED)} is a 
    sentence $\delta$ of the following form:
    \begin{equation}
    \label{eqn:epistemic-dependency}
        \begin{array}{l}
        \forall \vec{x}_1,\vec{x}_2 (K q_b(\vec{x}_1,\vec{x}_2) \rightarrow K q_h(\vec{x}_2))
        \end{array}
    \end{equation}
    where $q_b(\vec{x}_1, \vec{x}_2)$ is a CQ with free variables $\vec{x}_1 \cup \vec{x}_2$, $q_h(\vec{x}_2)$ is a CQ with free variables $\vec{x}_2$, and $K$ is a modal operator. The variables $\vec{x}_2$ are called the frontier variables of $\delta$.
\end{definition}

Intuitively, an ED of form (\ref{eqn:epistemic-dependency}) should be read as follows: if the sentence $q_b(\vec{c}_1, \vec{c}_2)$ is \emph{known} to hold, then the sentence $q_h(\vec{c}_2)$ is \emph{known} to hold, for any ground substitutions $\vec{c}_1$ and $\vec{c}_2$ for $\vec{x}_1$ and $\vec{x}_2$, respectively. 
More formally, we define when an FO theory $\Phi$ \emph{satisfies} an ED $\delta$, denoted $\Phi \modelseql \delta$.
To this aim, we consider the set $E$ of all FO models of $\Phi$, and say that $\Phi \modelseql \delta$ if, for every ground substitutions $\vec{c}_1$ for $\vec{x}_1$ and $\vec{c}_2$ for $\vec{x}_2$, the fact that $\eval( q_b(\vec{c}_1,\vec{c}_2),\I)$ is true for every $\I \in E$ implies that $\eval(q_h(\vec{c}_2),\I)$ is true for every $\I \in E$. 

We say that $\Phi$ \emph{satisfies} a policy $\P$ (denoted $\Phi \modelseql \P$) if $\Phi$ satisfies $\delta$, for each $\delta \in \P$.
We remark that, as already said, EDs of the form (\ref{eqn:epistemic-dependency}) have been originally introduced in~\cite{CoLe20}, although in a slightly more general form, to express integrity constraints in ontology-based data management.
Then, the notion of ED satisfaction defined above is essentially as in~\cite{CoLe20}. 


\begin{example}
\label{ex:exampleED}    

    Suppose that company \comA wants to share certain user-profiling data with a company \comB for targeted advertising. 
    This is not allowed in general, but only in some countries with a special regulation that enables sharing based on the users' consent. 
    \comA may use the following ED in the policy to enable \comB to access only data compliant with the above requirements:
    \begin{tabbing}
        \indent $\delta_4 = \forall x,y \big( $\=$K \profiledActivity(x,y)\rightarrow$ \\
        \> $ K\,\exists z(\citOf(x,z) \land \SR(z) \land \consent(x)) \big)$
    \end{tabbing}
    In the rule, $\profiledActivity$ associates a user with her profiling-data, $\citOf$ relates a user to the country of which she is a citizen, $\SR$ denotes countries with special regulation, and $\consent$ denotes users who have given their consent.  

    Suppose that \comA also wants \comB not to be able to associate a user with her real identity, and that this is possible by collecting the person's name and her date of birth at the same time. To this aim, \comA also specifies the following ED:
   \begin{tabbing}
        \indent $\delta_5 = \forall x,y,z \big(K ( \name(x,y) \land \dateBirth$\=$( x,z) ) \rightarrow K\bot\big)$ 
    \end{tabbing}
    \vspace{-1.5em}\qedex
\end{example}


\smallskip

\noindent \textbf{CQ-censors.} As already said, censors are used to enforce confidentiality on an ontology coupled with a data protection policy. Among various definitions of censors proposed in the literature, we adapt here the one investigated in~\cite{LeRS19} to properly deal with policies constituted by sets of EDs. To this aim, it is convenient to first define CQE instances in our novel epistemic CQE framework.

\begin{definition}[CQE instance]
\label{def:cqe-instance}
    An $\LT$ \emph{CQE instance} is a triple $\E=\tup{\T,\A,\P}$, where $\T$ is a TBox in the DL $\LT$, $\A$ is an ABox such that $\T\cup\A$ 
    has at least one model, and $\P$ is a policy (i.e., a finite set of EDs) such that $\T \modelseql \P$.
\end{definition}

Hereinafter, if the DL $\LT$ is not specified, we implicitly intend any possible DL. 

The notion of (optimal) CQ-censors is then as follows.

\begin{definition}[(optimal) CQ-censor]
\label{def:cq-censor}
    A \emph{CQ-censor} of a CQE instance $\E=\tup{\T,\A,\P}$ is a subset $\C$ of $\bcqclosure(\T\cup\A)$ such that $\T\cup\C\modelseql\P$.
    
    An \emph{optimal CQ-censor} of a CQE instance $\E=\tup{\T,\A,\P}$ is a CQ-censor of $\E$ such that there exists no CQ-censor $\C'$ of $\E$ such that $\C'\supset\C$.
    We denote by $\optcqcens(\E)$ the set of optimal CQ-censors of $\E$.
\end{definition}

\smallskip

\noindent \textbf{Entailment.} As in~\cite{LeRS19}, in this paper CQE amounts to reason over all the possible optimal censors, according to the following definition of \cqe-entailment.

\begin{definition}[\cqe-entailment]
\label{def:cqe-entailment}
    A CQE instance $\E=\tup{\T,\A,\P}$ \emph{Skeptically-entails a BUCQ $q$ under CQ-Censors} (in short, \emph{\cqe-entails} $q$), denoted by $\E\modelscqe q$, if $\T\cup\C\models q$ for every $\C\in\optcqcens(\E)$.
\end{definition}

We also consider the following sound approximation of \cqe-entailment.


\begin{definition}[\icqe-entailment]
\label{def:icqe-entailment}
A CQE instance $\E=\tup{\T,\A,\P}$ \emph{entails a BUCQ $q$ under the Intersection of CQ-Censors} (in short, \emph{\icqe-entails $q$}), denoted by $\E\modelsicqe q$, if $\T\cup\intcens\models q$, where $\intcens=\bigcap_{\C\in\optcqcens(\E)} \C$.
\end{definition}

\begin{example}\label{ex:bcq_entailment}
    Consider the CQE instance $\E=\tup{\T,\A,\P}$, where $\T=\emptyset$, $\P=\{\delta_4,\delta_5\}$, with $\delta_4$ and $\delta_5$ being the EDs illustrated in Example~\ref{ex:exampleED}, and the ABox $\A$ is as follows:
    \begin{tabbing}
        $\A=\{$\=$
        \profiledActivity(\pA,\actA),
        \consent(\pA),
        \citOf(\pA,n_1),
        $\\\>$
        \SR(n_1),
        \name(\pA,\ann),
        \dateBirth(\pA,\dA),
        $\\\>
        $\profiledActivity(\pB,\actB),
        \citOf(\pB,n_1)
        \}$,
    \end{tabbing}
    where $n_1 \in \labeledNullSet$ while all the other terms used in $\A$ are constants in $\individualSet$. Now, consider the following four BCQs:\\
    \indent $q_1=\exists y\, (\profiledActivity(\pA,\actA) \land \citOf(\pA,y) \land \SR(y))$\\
    \indent $q_2=\profiledActivity(\pB,\actB)$\\
    \indent $q_3=\exists y\, \profiledActivity(y,\actB)$\\
    \indent $q_4=\profiledActivity(\pA,\actA) \land \name(\pA,\ann)$\\
    %
    For $X \in \{\cqe,~\icqe\}$, one can verify that $\E \models_X q_1$ because $\pA$ gave the consent and she is a citizen of some country ($n_1$) with special regulation. Conversely, one can see that $\E \not\models_X q_2$ because $\pB$ did not give the consent. Nevertheless, it is easy to verify that $q_3 \in \C$ for each optimal CQ-censor $\C$ of $\E$, and therefore $\E \models_X q_3$. Finally, we have that $\E \not\models_X q_4$ because there exists an optimal CQ-censor $\C$ of $\E$ such that $\dateBirth(\pA,\dA) \in \C$, thus implying that $\name(\pA,\ann) \not \in \C$ (otherwise $\delta_5$ would be violated). \qedex
\end{example}

In the above example, note that \cqe\nobreakdash- and \icqe-entailment coincide for all queries. As shown in the subsequent result, this always holds in the case of entailment of BCQs.
%
\begin{theorem}
\label{thm:cqe-icqe}
For every CQE instance $\E=\tup{\T,\A,\P}$ and for every BCQ $q$, we have that $\E\modelscqe q$ iff $\E\modelsicqe q$.
\end{theorem}
\begin{proof}
First, it is easy to verify that, for every $\C$ that is an optimal CQ-censor of $\E$, and for every BCQ $q$, $\T\cup\C\models q$ iff $q\in\C$; 
moreover, $\T\cup\intcens\models q$ iff $q\in\intcens$.
Now, let $q$ be a BCQ. If $\E\modelscqe q$, then $q$ belongs to all the optimal CQ-censors of $\E$, and thus $q\in\intcens$, which implies that $\E\modelsicqe q$. Conversely, if $\E\not\modelscqe q$, then there exists an optimal CQ-censor of $\E$ that does not contain $q$, hence $q\not\in\intcens$, which implies that $\E\not\modelsicqe q$. 
\end{proof}

On the other hand, the next example shows that the same result does not hold for entailment of BUCQs.

\begin{example}
    Recall Example~\ref{ex:bcq_entailment}, and consider the BUCQ $q=\name(\pA,\ann) \vee \dateBirth(\pA,\dA)$. While we have that $\E \modelscqe q$, because either $\name(\pA,\ann) \in \C$ or $\dateBirth(\pA,\dA) \in \C$ holds for every $\C\in\optcqcens(\E)$, it is easy to see that $\E \not \modelsicqe q$ as neither $\name(\pA,\ann)$ nor $\dateBirth(\pA,\dA)$ belong to $\intcens=\bigcap_{\C\in\optcqcens(\E)} \C$.
    \qedex
\end{example}







%% file: 4-indistinguishability.tex
\section{Confidentiality preservation}
\label{sec:indistinguishability}

In this section, we investigate the notion of confidentiality used in this paper. 
We adopt a similar approach to the one described in~\cite{BiWe08} for relational databases. Intuitively, under such an approach an entailment semantics preserves confidentiality if, for every CQE instance $\tup{\T,\A,\P}$ and every finite set $\Q$ of queries, the answers to such queries are the same as if they were obtained from another CQE instance $\tup{\T,\A',\P}$
such that $\T \cup \A' \modelseql \P$.



We now describe this property formally.
First, for a TBox $\T$, a policy $\P$, two ABoxes $\A$ and $\A'$, a set $\Q$ of BUCQs, and $X \in \{\cqe,~\icqe\}$, we say that $\A$ and $\A'$ are \emph{$\Q$-indistinguishable for $X$-entailment with respect to $(\T,\P)$} if, for every $q\in\Q$, we have that $\tup{\T,\A,\P}\models_{X} q$ iff $\tup{\T,\A',\P}\models_{X} q$. 
\begin{definition}
    Given a query language $\LQ\subseteq\BUCQ$ and a DL $\LT$, for $X \in \{\cqe,~\icqe\}$, we say that \emph{$X$-entailment preserves confidentiality for $\LQ$ in $\LT$} if, for every $\LT$ CQE instance $\tup{\T,\A,\P}$, and for every finite set $\Q\subseteq\LQ$, there exists an ABox $\A'$ such that $(i)$ $\T\cup\A'\modelseql\P$ and $(ii)$ $\A$ and $\A'$ are $\Q$-indistinguishable for X-entailment w.r.t.\ $(\T,\P)$.
\end{definition}

For both \cqe-entailment and \icqe-entailment, we now investigate confidentiality-preservation for $\BUCQ$ in $\dlliter$. 

Hereinafter, with a slight abuse of notation, given a policy $\P$, we denote by $\maxlengthatoms(\P)$ the maximum length (number of atoms) of a CQ occurring within the scope of the $K$ operator in $\P$.

%
\begin{proposition}
\label{pro:indistinguishability-cqe}
    Let $\E=\tup{\T,\A,\P}$ be a $\dlliter$ CQE instance. 
    For every integer $k>0$, there exists
    a $\dlliter$ CQE instance $\E'=\tup{\T,\A',\P}$ such that $(i)$ $\T\cup\A'\modelseql\P$ and $(ii)$ for every $q\in\BCQk$, $\E\modelscqe q$ iff $\E'\modelscqe q$.
\end{proposition}
\ifjournal
\else 
\begin{proofsk}
    Let $\A'$ be the ABox isomorphic to the finite set of BCQs $\{ q\in\BCQh \mid \E\modelscqe q \}$, where $h=\max(k,\maxlengthatoms(\P))$. It is easy to verify that $\T\cup\A'\modelseql\P$, from which it follows that $\E\modelscqe q$ iff $\E'\modelscqe q$ for every $q\in\BCQk$. 
\end{proofsk}
\medskip
\fi


With the above property at hand, we can prove that \cqe-entailment preserves confidentiality for $\BCQ$ in $\dlliter$. We show, however, that the same does not hold for $\BUCQ$.


\begin{theorem}
\label{thm:indistinguishability-cqe}
    \cqe-entailment preserves confidentiality for $\BCQ$ in $\dlliter$, whereas it does not preserve confidentiality for $\BUCQ$ in $\dlliter$.
\end{theorem}
\begin{proof}
    The first statement is an easy consequence of Proposition~\ref{pro:indistinguishability-cqe}, when we assume that $k$ is the maximum length of a BCQ in $\Q$.
    For the second statement, we give a counterexample. 
    Let $\E = \tup{\T,\A,\P}$, where $\T=\emptyset$, $\A = \{C_1(o), C_2(o)\}$ and $\P = \{ \forall x (K(C_1(x) \land C_2(x)) \ra K \bot) \}$. Consider also the BUCQ $q = C_1(o) \lor C_2(o)$. It is easy to see that no ABox $\A'$ is such that $\T\cup\A'\modelseql\P$, and $\A$ and $\A'$ are $\{q\}$-indistinguishable for \cqe-entailment w.r.t.\ $(\T,\P)$.
\end{proof}



The above proof also shows that \cqe-entailment does not preserve confidentiality for $\BUCQ$ in $\dlliter$ even when EDs are restricted to be acyclic (see Definition~\ref{def:acyclic}).

However, it turns out that confidentiality in the case of BUCQs in $\dlliter$ is preserved by \icqe-entailment.

\begin{proposition}
\label{pro:indistinguishability-icqe}
    Let $\E=\tup{\T,\A,\P}$ be a $\dlliter$ CQE instance. 
    For every integer $k>0$, there exists
    a $\dlliter$ CQE instance $\E'=\tup{\T,\A',\P}$ such that $(i)$ $\T\cup\A'\modelseql\P$ and $(ii)$ for every $q\in\BUCQk$, $\E\modelsicqe q$ iff $\E'\modelsicqe q$.
\end{proposition}
\ifjournal
\else 
\fi 
%

\medskip

From the above property, we get the following result.
\begin{theorem}
\label{thm:indistinguishability-icqe}
\icqe-entailment preserves confidentiality for $\BUCQ$ in $\dlliter$. 
\end{theorem}

%% file: 5.1-results-cyclic.tex
\section{Algorithms and complexity results}
\label{sec:results}

\newcommand{\varprop}{\textit{PV}}
\newcommand{\expand}{\textit{Expand}}
\newcommand{\unfold}{\textit{Unfold}}

%
In this section, we analyze the data complexity of \cqe- and \icqe-entailment of B(U)CQs for $\dlliter$ CQE instances. 

We study the decision problems associated with the query answering problem under \cqe- and \icqe-entailment. Specifically, we consider the following recognition problem $X$-\decprob[$\LQ$], which is parametric w.r.t.\ a Boolean query language $\LQ$ and $X \in \{\cqe,~\icqe\}$:
\begin{quote}
    \textbf{Input:} A $\dlliter$ CQE instance $\E=\tup{\T, \A,\P}$, a Boolean query $q \in \LQ$\\
    \textbf{Question:} Does $\E\models_X q$?
\end{quote}

We are interested in the data complexity~\cite{Vard82} version of the above problem, which is the complexity where only the ABox $\A$ is regarded as the input while all the other components are assumed to be fixed.

\subsection{Arbitrary policies} 
\label{sec:arbitrary}

We start by analyzing the complexity of SC-entailment of BCQs and BUCQs.

\begin{lemma}
\label{lem:cqe-entailment-bcq-lb}
    \cqe-\decprob[\BCQ]
    is \coNP-hard in data complexity.
\end{lemma}

\ifjournal\begin{proof}
\else\begin{proofsk}
\fi
    We show a reduction of 3-CNF, a well-known NP-hard problem, to the complement of \cqe-\decprob[\BCQ].
    Let $\T$ be the empty TBox, and let $\P$ contain the following EDs:
    \[
    \begin{array}{r@{}l}
    \forall x,y,v,z\,
    \big( K (&C_1(x,y)\wedge V_1(x,v)\wedge V(y,v) \wedge \\&\quad N(x,z)\wedge S(x) ) \rightarrow K S(z) \big) \\
    \forall x,y,v,z\,
    \big( K (&C_2(x,y)\wedge V_2(x,v)\wedge V(y,v) \wedge \\&\quad N(x,z)\wedge S(x) ) \rightarrow K S(z) \big) \\
    \forall x,y,v,z\,
    \big( K (&C_3(x,y)\wedge V_3(x,v)\wedge V(y,v) \wedge \\&\quad N(x,z)\wedge S(x) ) \rightarrow K S(z) \big) \\
    \forall x \, \big(K (&V(x,f)\wedge V(x,t)) \rightarrow K \bot \big)
    \end{array}
    \]
    
    Given a 3-CNF formula $\phi$ with $m$ clauses, we represent every clause of $\phi$ in the ABox $\A$ through the roles $C_1,C_2,C_3,V_1,V_2,V_3$: e.g., if the $i$-th clause of $\phi$ is $\neg a \vee b \vee \neg c$, we add to $\A$ the assertions $C_1(i,a),C_2(i,b),C_3(i,c),V_1(i,f),V_2(i,t),V_3(i,f)$. Moreover, $\A$ contains the assertions
    \[ 
    \begin{array}{l}
    \{ V(a,f),V(a,t) \mid a\in\varprop(\phi)\}\: \cup \\
    \{ S(i), N(i,i+1) \mid 1\leq i\leq m \} 
    \end{array}
    \]
    where $\varprop(\phi)$ are the propositional variables of $\phi$.
\ifjournal
\else 
    It can be verified that $\phi$ is satisfiable iff $\tup{\emptyset,\A,\P}\not\modelscqe S(1)$.
\fi 
\ifjournal\else\end{proofsk}\smallskip\fi 
\ifjournal\end{proof}\fi

In order to prove a matching coNP upper bound for \cqe-\decprob[\BUCQ], we introduce a notion of consequences of a set of BCQs $\C$ with respect to a TBox $\T$ and a policy $\P$.

\begin{definition}
\label{def:kcons}
Let $\T$ be a $\dlliter$ TBox, $\P$ a policy and $\C\subseteq\BCQ$.
We define $\kcons(\T,\P,\C)$ as the smallest set $\S\subseteq\BCQ$ such that:
\begin{itemize}
    \item[(i)] $\C\subseteq\S$;
    \item[(ii)] for every ED 
    $\forall \vec{x}_1,\vec{x}_2 (K q_b(\vec{x}_1,\vec{x}_2) \rightarrow K q_h(\vec{x}_2))$ in $\P$ and ground substitutions $\vec{t}_1$ and $\vec{t}_2$ 
    for $\vec{x}_1$ and $\vec{x}_2$, respectively, if $\T\cup\S\models q_b(\vec{t}_1,\vec{t}_2)$ then $q_h(\vec{t}_2)\in\S$.
\end{itemize}
\end{definition}

It can be immediately verified that $\kcons(\T,\P,\C)$ can be computed in polynomial time w.r.t.\ the size of $\C$.

The following property can be easily derived from the previous definition and the definition of optimal CQ-censor. 

\begin{lemma}
\label{lem:kcons}
    Let $\E=\tup{\T,\A,\P}$ be a $\dlliter$ CQE instance. For every $\C\subseteq\BCQ$, there exists an optimal CQ-censor of $\E$ containing $\C$ iff $\kcons(\T,\P,\C)\subseteq\bcqclosure(\T\cup\A)$. 
\end{lemma}

We are now ready to provide an algorithm for checking \cqe-entailment of BUCQs. 
\vspace{-0.7em}
\begin{algorithm}[ht!]
\caption{\cqe-Entails($\E,q$)}
\label{alg:cqe-entails}
    \Input{$\dlliter$ CQE instance $\E=\tup{\T,\A,\P}$, BUCQ $q$}
    $k\gets\max(\maxlength(q),\maxlengthatoms(\P))$;\\
    \lIf{\upshape there exists $\C\subseteq\bcqkclosure(\T\cup\A)$ such that
        \begin{itemize}[noitemsep,topsep=0pt]
            \item[$(i)$] $\T\cup\C\modelseql\P$ and
            \item[$(ii)$] $\T\cup\C\not\models q$ and
            \item[$(iii)$] for every $q'\in \bcqkclosure(\T\cup\A) \setminus \C$,\\
                $\kcons(\T,\P,\C\cup\{q'\})\not\subseteq\bcqkclosure(\T\cup\A)$
        \end{itemize}
    }{\Return $false$}
    \Return $true$;
\end{algorithm}
\FloatBarrier 

\begin{example}
Let $\E=\tup{\T,\A,\P}$, where $\T=\{A \ISA D\}$, $\P=\{\forall x(K(D(x) \wedge C(x)) \rightarrow K\bot),~ \forall x(KB(x) \rightarrow KA(x))\}$, and $\A=\{A(o),B(o),C(o)\}$. 

Let now $\C$ be a maximal subset of $\BCQ_2\textit{-Cons}(\T \cup \A)$  such that 
$\T\cup\C\not\models A(o)\vee B(o)\vee D(o)$.
One can see that, for the BCQ $q=B(o)$, \cqe-Entails($\E,q$) returns false as $\C$ satisfies conditions $(i)$, $(ii)$, and $(iii)$ of the algorithm.
In particular, $\kcons(\T,\P,\C\cup\{q\})$ contains $B(o)$, $A(o)$, $D(o)$, and $\bot$, hence it is not a subset of $\BCQ_2\textit{-Cons}(\T \cup \A)$.
\qedex
\end{example}

%

\begin{lemma}
\label{lem:cqe-entails}
    Let $\E=\tup{\T,\A,\P}$ be a $\dlliter$ CQE instance, and let $q$ be a BUCQ. The algorithm \cqe-Entails($\E,q$) returns true iff $\E\modelscqe q$.
\end{lemma}
\ifjournal
\else 
\begin{proofsk}
    First, using Lemma~\ref{lem:kcons}, it is easy to derive that a set of BCQs $\C$ that satisfies conditions (i), (ii), and (iii) exists iff there exists an optimal CQ-censor of $\E$ containing $\C$ and not containing $q$ (and therefore iff $\E\not\modelscqe q$).
    Then,
    it can be shown that, for a $\dlliter$ TBox $\T$ and a set $\C$ of BCQs that is closed under subqueries (i.e., for every BCQ $q\in\C$, $\C$ contains all the subqueries of $q$), $\T\cup\C\models q$ iff $\T\cup\C_h\models q$, where $h=\maxlength(q)$ and $\C_h=\C\cap\BCQk$. This is the key property that allows for proving that, if a set of BCQs $\C$ satisfying conditions (i), (ii), and (iii) exists, then there exists a set of BCQs $\C\subset\bcqkclosure(\T\cup\A)$ that satisfies the above conditions, which implies the correctness of the algorithm.
\end{proofsk}
\fi 

\medskip

The next theorem follows from Lemma~\ref{lem:cqe-entailment-bcq-lb}, Lemma~\ref{lem:cqe-entails}, and from the fact that the previous algorithm can be executed in nondeterministic polynomial time in data complexity.

\begin{theorem}
\label{thm:cqe-entailment-bcq-ub}
    \cqe-\decprob[\BUCQ] is \coNP-complete in data complexity.
\end{theorem}

We now give a second algorithm, which makes use of the above algorithm \cqe-Entails to check \icqe-entailment.

\vspace{-0.7em}
\begin{algorithm}[ht!]
\caption{\icqe-Entails($\E,q$)}
\label{alg:icqe-entails}
    \Input{$\dlliter$ CQE instance $\E=\tup{\T,\A,\P}$, BUCQ $q$}
    \ForEach{\upshape BCQ $q'\in q$}{%
        \lIf{\upshape \cqe-Entails($\E,q'$)}{\Return $true$}
    }
    \Return $false$;
\end{algorithm}
\FloatBarrier 

\begin{lemma}
\label{lem:icqe-entails}
    Let $\E=\tup{\T,\A,\P}$ be a $\dlliter$ CQE instance 
    and let $q$ be a BUCQ. The algorithm \icqe-Entails($\E,q$) returns true iff $\E\modelsicqe q$.
\end{lemma}
\begin{proof}
    First, from Definition~\ref{def:icqe-entailment}, if $\T$ is a $\dlliter$ TBox, then $\E\modelsicqe q$ iff there exists a BCQ $q'\in q$ such that $\E\modelsicqe q'$. Then, by Theorem~\ref{thm:cqe-icqe}, $\E\modelsicqe q'$ iff $\E\modelscqe q'$. 
    Therefore, by Lemma~\ref{lem:cqe-entails} the thesis follows. 
\end{proof}

From Lemma~\ref{lem:cqe-entailment-bcq-lb}, Lemma~\ref{lem:icqe-entails}, Theorem~\ref{thm:cqe-icqe}, and from the fact that the algorithm \cqe-Entails($\E,q$) can be executed in nondeterministic polynomial time in data complexity, we obtain:

\begin{theorem}
\label{thm:icqe-entailment-bcq-ub}
    \icqe-\decprob[\BUCQ] is \coNP-complete in data complexity.
\end{theorem}

%% file: 5.2-results-acyclic.tex
\subsection{Acyclic policies} 
\label{sec:acyclic}

Given the intractability results for the set of all $\dlliter$ CQE instances presented above, in this section we focus on a subclass of $\dlliter$ CQE instances, whose policies enjoy an acyclicity property. 

First, we extend the notion of first-order rewritability defined over ground ABoxes to the case of 
ABoxes with labeled nulls and to the problems of SC-entailment and IC-entailment of BUCQs.
Given an ABox $\A$, we define the FO interpretation $\I_\A$ over the predicates $\conceptSet\cup\roleSet$ plus the additional new concept $\constpred$, and the constants $\individualSet\cup\labeledNullSet$ (i.e.\ in $\I_\A$ we consider the symbols from $\labeledNullSet$ as ordinary constants):
\begin{itemize}\itemsep -1pt
\item $\Delta^{\I_\A}=\individualSet\cup\labeledNullSet$;
\item $a^{\I_\A}=a$ for every $a\in\individualSet\cup\labeledNullSet$;
\item for every concept name $C$, $C^{\I_\A}=\{a \mid C(a)\in \A\}$;
\item for every role name $R$, $R^{\I_\A}=\{(a,b) \mid R(a,b)\in \A\}$;
\item $\constpred^{\I_\A}=\individualSet$.
\end{itemize}

Given a TBox $\T$, a policy $\P$ and a BUCQ $q$, and $X \in \{\cqe,~\icqe\}$, we say that a FO sentence $q'$ is a \emph{first-order rewriting of $X$-entailment of $q$ for $\T$ and $\P$} if, for every ABox $\A$, $\tup{\T,\A,\P}\models_X q$ iff $\eval(q',\I_\A)$ is true.


Our goal now is to define an algorithm that, for a $\dlliter$ TBox $\T$ and a policy $\P$, is able to construct a first-order rewriting of SC-entailment of $q$ for $\T$ and $\P$. This is not possible in general, given the coNP-completeness result provided by Theorem~\ref{thm:icqe-entailment-bcq-ub}. Therefore, we now define the subclass of policies that are acyclic for a $\dlliter$ TBox.

Given a $\dlliter$ TBox $\T$ and a policy $\P$, the \emph{dependency graph} of $\T$ and $\P$, denoted by $G(\T,\P)$, is the directed graph defined as follows: (i) the set of nodes of $G(\T,\P)$ is the set of predicates occurring in $\T\cup\P$; (ii) there is a P-edge from node $p_1$ to node $p_2$ in $G(\T,\P)$ if and only if there exists an epistemic dependency of the form (\ref{eqn:epistemic-dependency}) in $\P$ such that $p_1$ occurs in $q_b$ and $p_2$ occurs in $q_h$; (iii) there is a T-edge from node $p_1$ to node $p_2$ in $G(\T,\P)$ if and only if there is a concept or role inclusion in $\T$ such that $p_1$ occurs in the left-hand side and $p_2$ occurs in the right-hand side of the inclusion. 




\begin{definition}
\label{def:acyclic}
Given a $\dlliter$ TBox $\T$ and a policy $\P$, we say that $\P$ is \emph{acyclic for $\T$} if there exists no cycle in $G(\T,\P)$ involving a P-edge.
\end{definition}

Informally, the graph $G(\T,\P)$ represents the logical dependencies between the predicates in $\T$ and $\P$: a P-edge (resp., a T-edge) from $p_1$ to $p_2$ means that predicate $p_1$ may have a direct implication on $p_2$ through $\P$ (resp., through $\T$). The notion of acyclicity defined above guarantees that, if $(p_1,p_2)$ is a P-edge in $G(\T,\P)$, then there is no path from $p_2$ to $p_1$, i.e.\ $p_2$ does not have any (direct or indirect) implication on $p_1$.



\newcommand{\KeyPosition}{\mathsf{KeyPosition}}
\newcommand{\hasPosition}{\mathsf{hasPosition}}
\newcommand{\worksIn}{\mathsf{worksIn}}
\newcommand{\PublicOffice}{\mathsf{PublicOffice}}
\newcommand{\collaborate}{\mathsf{collaborate}}
\newcommand{\SecService}{\mathsf{SecService}}
\newcommand{\pp}{\mathsf{p_1}}
\newcommand{\pb}{\mathsf{p_2}}
\begin{example}
The following ED aims to hide the hierarchical structure of an organization unless it is public.
\begin{tabbing}
    \indent $\delta_6 = \forall x,y ($\=$K \hasPosition(x,y)  \rightarrow$ \\ \> $K\,\exists z(\worksIn(x,z) \land \PublicOffice(z))))$
\end{tabbing}
Moreover, we want to hide the fact that a person collaborates with a secret service unless she holds an important position ($\KeyPosition$), for example, she is the director. The following ED achieves this goal:
\begin{tabbing}
    \indent $\delta_7 = \forall x ($\=$K\,\exists y(  \collaborate(x,y) \land \SecService(y))\rightarrow$ \\\> $K\,\exists z( \hasPosition(x,z) \land \KeyPosition(z)))$
\end{tabbing}
\noindent Let the policy $\P$ be $\P=\{\delta_6,\delta_7\}$. For the empty TBox $\T=\emptyset$, one can see that $\P$ is acyclic for $\T$. Conversely, for the $\dlliter$ TBox $\T=\{ \worksIn \ISA \collaborate \}$, $\P$ is not acyclic for $\T$, since there is a cycle in $G(\T,\P)$ constituted by the P-edges $(\collaborate,\hasPosition)$ and $(\hasPosition,\worksIn)$ and the T-edge $(\worksIn,\collaborate)$.
\qedex
\end{example}







With this notion in place, we can now describe the decision problem we are going to study. 
Specifically, for $X \in \{\cqe,~\icqe\}$, we consider the recognition problem $X$-\acdecprob[$\LQ$], which is parametric w.r.t.\ a Boolean query language $\LQ$. $X$-\acdecprob[$\LQ$] is defined exactly as $X$-\decprob[$\LQ$] except that the input $\dlliter$ CQE instances $\tup{\T,\A,\P}$ are such that the policy $\P$ is acyclic for $\T$.

From now on, given a set of CQs $\Q$,\footnote{W.l.o.g.\ we assume that the sets of existentially quantified variable symbols used by the CQs in $\Q$ are pairwise disjoint.} we denote by $\andcqs(\Q)$ the CQ
$
\exists\vec{y} \big(\bigwedge_{\exists\vec{z}\phi\in\Q} \phi \big)
$,
where $\vec{y}$ is a sequence of all the existentially quantified variables occurring in $\Q$.

Given a TBox $\T$, a policy $\P$ and a CQ $q(\vec{x})$, we denote by $\phikcons(\T,\P,q(\vec{x}))$ the CQ (with free variables $\vec{x}$) $\andcqs(\kcons(\T,\P,\{q(\vec{x})\})$,
where $\kcons(\T,\P,\{q(\vec{x})\})$ is as in Definition~\ref{def:kcons}, considering the free variables in $\vec{x}$ as new constant symbols.

\begin{lemma}
\label{lem:rewriting-consistency}
    Let $\E=\tup{\T,\A,\P}$ be a $\dlliter$ CQE instance,
    let $\P$ be an acyclic policy for $\T$, and let $q(\vec{x})$ be a CQ.
    For every ground substitution $\vec{c}$ for $\vec{x}$, 
    there exists an optimal CQ-censor of $\E$ that contains the BCQ $q(\vec{c})$ if and only if $\eval(q_r(\vec{c}),\I_\A)$ is true, where $q_r(\vec{x})=\perfectref(\phikcons(\T,\P,q(\vec{x})),\T)$.
\end{lemma}


Given a BCQ $q$, we say that a BCQ $q'$ is a \emph{clash for $q$ in $\E$} if there exists an optimal CQ-censor of $\E$ containing $q'$ and there exists no optimal CQ-censor of $\E$ containing both $q$ and $q'$.
Given a BUCQ $q$, we say that a BCQ $q'$ is a clash for $q$ in $\E$ if for every $q''\in q$, $q'$ is a clash for $q''$ in $\E$.

Now, let $q$ be a BUCQ and let $q'(\vec{x})$ be a CQ. 
We denote by $\clash(q,q'(\vec{x}),\T,\P)$ the following FO formula (with free variables $\vec{x}$):
\[
\begin{array}{l}
\perfectref(\phikcons(\T, \P, q'(\vec{x})),\T)\: \wedge \\
\displaystyle
\big(
\bigwedge_{q_i\in q}\neg\perfectref(\phikcons(\T, \P, \andcqs(\{q'(\vec{x}),q_i)\}),\T)
\big)
\end{array}
\]
Using Lemma~\ref{lem:rewriting-consistency}, we are able to prove the following property.

\begin{lemma}
\label{lem:clash}
 Let $\E=\tup{\T,\A,\P}$ be a $\dlliter$ CQE instance, let $q\in \BUCQ$, let $q'(\vec{x})$ be a CQ,
 and let $q_i\in\bcqclosure(\T\cup\A)$ for every $q_i\in q$.
 Then, for every ground substitution $\vec{c}$ for $\vec{x}$, 
 $q'(\vec{c})$ is a clash for $q$ in $\E$ iff $\eval(q_{cl}(\vec{c}),\I_\A)$ is true, where $q_{cl}(\vec{x})=\clash(q,q'(\vec{x}),\T,\P)$.
\end{lemma}

It is now possible to show that, in the case of $\dlliter$ CQE instances in which $\P$ is an acyclic policy for $\T$, if a clash for a BUCQ $q$ exists, then there exists a clash for $q$ whose length depends only on the size of $\P\cup\T\cup\{q\}$.

\begin{lemma}
\label{lem:clash-dlliter-acyclic}
    Let $\E=\tup{\T,\A,\P}$ be a $\dlliter$ CQE instance such that $\P$ is acyclic for $\T$, and let $q$ be a BUCQ such that $q_i\in\bcqclosure(\T\cup\A)$ for every $q_i\in q$. Then, $\E\modelscqe q$ iff there exists no BCQ $q'$ such that $q'$ is a clash for $q$ in $\E$ and $\cqlength(q')\leq \ell$, where $\ell=m\cdot k^h$, $m$ is the number of BCQs in $q$, $k=\maxlengthatoms(\P)$, and $h$ is the number of EDs in $\P$.
\end{lemma}

We can now define an FO sentence and then prove that it provides a first-order rewriting of SC-entailment of BUCQs.


\begin{definition}
\label{def:sc-entailed}
Let $\T$ be a $\dlliter$ TBox, let $\P$ be an acyclic policy for $\T$, and let $q \in \BUCQ$. We define the FO sentence $\censentailed(q,\T,\P)$ as follows:

\begin{small}
\[
\begin{array}{l}
\displaystyle
\bigvee_{q^p\in\wp^-(q)}
\Big(
 \bigwedge_{q_i\in q^p}\perfectref(q_i,\T)\: \wedge 
 \bigwedge_{q_i\in q\setminus q^p}\neg\perfectref(q_i,\T)\: \wedge \\[1mm]
\displaystyle
\qquad\;\;
\bigwedge_{q'(\vec{x})\in \Q}  
\neg \big( \exists \vec{x} \,(\clash(q^p,q'(\vec{x}),\T,\P) \wedge \bigwedge_{x\in\vec{x}} Ind(x) ) \big)
\Big)
\end{array}
\]
\end{small}

\noindent 
with $\wp^-(q)=\wp(q)\setminus\{\emptyset\}$, where $\wp(q)$ is the powerset of $q$, 
$\Q$ is the set of CQs defined over the predicates and constants occurring in $\{q\} \cup\P \cup \T$ and whose maximum length is $m\cdot k^h$,
$m$ is the number of BCQs in $q$, $h$ is the number of EDs in $\P$, and $k=\maxlengthatoms(\P)$.
\end{definition}

Based on Lemma~\ref{lem:clash} and Lemma~\ref{lem:clash-dlliter-acyclic}, we are able to prove the following crucial property for $\censentailed(q,\T,\P)$.

\begin{lemma}
\label{lem:rewriting-cqe-entailment}
    Let $\E=\tup{\T,\A,\P}$ be a $\dlliter$ CQE instance such that 
    $\P$ is an acyclic policy for $\T$, and let $q \in \BUCQ$. Then, 
    $\censentailed(q,\T,\P)$ is a first-order rewriting of SC-entailment of $q$ for $\T$ and $\P$.
\end{lemma}

The above first-order rewritability property of SC-entailment of BUCQs immediately implies the next result.

\begin{theorem}
\label{thm:cqe-entailment-full}
    \cqe-\acdecprob[$\BUCQ$] is in $\aczero$ in data complexity.
\end{theorem}


Finally, given a BUCQ $q$, we define the sentence $\icensentailed(q,\T,\P)$ as follows:
\[
\begin{array}{l}
\bigvee_{q_i\in q} \censentailed(q_i,\T,\P)
\end{array}
\]

It is then easy to prove the analogous of Lemma~\ref{lem:rewriting-cqe-entailment} (and Theorem~\ref{thm:cqe-entailment-full}) for $\icensentailed(q,\T,\P)$ and IC-entailment.

\begin{lemma}
\label{lem:rewriting-icqe-entailment}
    Let $\E=\tup{\T,\A,\P}$ be a $\dlliter$ CQE instance such that 
    $\P$ is an acyclic policy for $\T$, and let $q \in \BUCQ$. Then, 
    $\icensentailed(q,\T,\P)$ is a first-order rewriting of IC-entailment of $q$ for $\T$ and $\P$.
\end{lemma}
\begin{proof}
The result follows immediately from Lemma~\ref{lem:rewriting-cqe-entailment} and from the fact that $\E\modelsicqe q$ iff there exists a BCQ $q_i\in q$ such that $\E\modelscqe q_i$ (the latter property easily follows from Definition~\ref{def:icqe-entailment} and Theorem~\ref{thm:cqe-icqe}).
\end{proof}

\begin{theorem}
\label{thm:icqe-entailment-full}
    \icqe-\acdecprob[$\BUCQ$] is in $\aczero$ in data complexity.
\end{theorem}

%% file: 6-conclusions.tex
\section{Conclusions}
\label{sec:conclusions}

The results given in this paper are summarized in Table~\ref{tab:results}. Beyond their theoretical connotation, our results for acyclic dependencies are particularly interesting for practical applications, since data complexity in these cases is the same as that for standard query answering over databases, and this paves the way for implementation through consolidated SQL technology. Moreover, the table shows that confidentiality is preserved in most of the cases that we have considered.

We posit that the epistemic nature of our framework makes it 
suited to being extended to incorporate user background knowledge, which can be modeled through appropriate epistemic formulas. The implications of this extension remain a subject for future investigation. Further possible development may explore ontology languages alternative to $\dlliter$, such as $\mathcal{EL}$~\cite{BaBL05} or the OWL~2 profiles~\cite{W3Crec-OWL2-Profiles}. Additionally, extending the framework to accommodate preferences on data to be censored while ensuring compliance to the policy, as in~\cite{CLMRS21}, and examining a dynamic context where censors filter responses based on previous answers, as explored in~\cite{BCLMRSS22}, are 
paths for further research. 

%% file: appendix.tex
\section*{Appendix}
\label{sec:appendix}

\subsection*{Auxiliary notions and results}

We recall the following property, which is a direct consequence of~\cite[Theorem~6]{CDLLR07b}.

\begin{proposition}
\label{pro:eqllite-ucq-eval}
    Let $\Phi,\Phi'$ be FO theories, and let $\phi$ be an ED. 
    If, for every subformula $Kq(\vec{x})$ occurring in $\phi$, and for every ground substitution $\vec{t}$ for $\vec{x}$, 
    $\Phi\models q(\vec{t})$ iff $\Phi'\models q(\vec{t})$, then $\Phi\modelseql\phi$ iff $\Phi'\modelseql\phi$.
\end{proposition}

%


We say that a set of BCQs $\C$ is \emph{closed under subqueries} if, for every $q\in\C$ and for every subquery\footnote{Given a BCQ $q$ of the form $\exists \vec{x}_1 (\alpha_1\wedge \ldots \wedge\alpha_n)$, a \emph{subquery} of $q$ is a BCQ $q'$ of the form $\exists \vec{x}_2 (\alpha_{i_1}\wedge \ldots \wedge\alpha_{i_m})$ such that, for every $j\in\{1,\ldots,m\}$, $1\leq i_j\leq n$, and $\vec{x}_2$ are the variables of $\vec{x}_1$ that occur in some $\alpha_{i_j}$. Informally, the subquery $q'$ is obtained from $q$ deleting some of its atoms.} $q'$ of $q$, we have $q'\in\C$.

\begin{proposition}
\label{pro:dlliter-bcqk-projection}
    Let $\T$ be a $\dlliter$ TBox, and let $\C$ be a set of BCQs closed under subqueries. For every BUCQ $q$, $\T\cup\C\models q$ iff $\T\cup\C_k\models q$, where $k=\maxlength(q)$ and $\C_k=\C\cap\BCQk$.
\end{proposition}
\begin{proof}
    We prove that, if $\T\cup\C\models q$ then $\T\cup\C_k\models q$ (the only-if direction is trivial). First, by Proposition~\ref{pro:rewriting-dlliter}, $\T\cup\C\models q$ iff $\C\models q'$, where $q'=\perfectref(q,\T)$. As said, we also have that $\maxlength(q)=\maxlength(q')$. Moreover, $\C\models q'$ iff there exists a BCQ $q''$ in $q'$ and a homomorphism $h$ mapping $q''$ into $\C$: now, since $\C$ is closed under subqueries, the existence of $h$ implies the existence of a homomorphism that maps $q''$ onto a subset of BCQs $\C'$ of $\C$ such that the length of each BCQ of $\C'$ is not greater than $\cqlength(q'')$ (which of course is not greater than $\maxlength(q')$). This immediately implies that $\C\models q'$ iff $\C_k\models q'$, which in turn implies the thesis.
\end{proof}


\smallskip
\subsection*{Proofs of theorems and lemmas}

\noindent\textbf{Proof of Proposition~\ref{pro:indistinguishability-cqe}}

Let $\Q$ be the finite set of BCQs $\{ q\in\BCQh \mid \E\modelscqe q \}$, where $h=\max(k,\maxlengthatoms(\P))$. We observe that $\Q \subseteq \bigcap_{\C\in\optcqcens(\E)} \C$.  Let also $\A'$ be the ABox isomorphic to $\Q$, i.e.\ the ABox obtained by $(i)$ creating, for each $q\in \Q$, a set $\A_q$ collecting the atoms occurring in $q$; $(ii)$ replacing distinct variables occurring in each $\A_q$ with distinct fresh labeled nulls; $(iii)$ representing $\A'$ as $\bigcup_{q\in \Q} \A_q$.

We now prove by contradiction that $\T\cup\A'\modelseql\P$, so let us suppose this does not hold. In this case, there would exist an ED $\tau\in\P$ of the form $\forall \vec{x}_1,\vec{x}_2 (K q_b(\vec{x}_1,\vec{x}_2) \rightarrow K q_h(\vec{x}_2))$ and two ground substitutions $\vec{c_1}$ and $\vec{c_2}$ for $\vec{x}_1$ and $\vec{x_2}$, respectively, such that $\T\cup\A'\models q_b(\vec{c}_1,\vec{c}_2)$ and $\T\cup\A'\not\models q_h(\vec{c}_2)$.
By construction of $\A'$ this holds only if $\T\cup\Q\models q_b(\vec{c}_1,\vec{c}_2)$ and $\T\cup\Q\not\models q_h(\vec{c}_2)$.
Since, as noticed, $\Q$ is contained in every CQ-censor for $\E$, then from $\T\cup\Q\models q_b(\vec{c}_1,\vec{c}_2)$ it follows by monotonicity that $\T\cup\C\models q_b(\vec{c}_1,\vec{c}_2)$ holds for every CQ-censor $\C$ for $\E$ (i.e.\ $\E\modelscqe q_b(\vec{c}_1,\vec{c}_2)$).
Moreover, from $\T\cup\Q\not\models q_h(\vec{c}_2)$ and by construction of $\Q$ it follows that $\E\not\modelscqe q_h(\vec{c}_2)$ (indeed, from $\E\modelscqe q_h(\vec{c}_2)$ we would have that $q_h(\vec{c}_2)\in\Q$, and so that $\T\cup\Q\modelscqe q_h(\vec{c}_2)$).

Since $\E\modelscqe q_b(\vec{c}_1,\vec{c}_2)$ and $\E\not\modelscqe q_h(\vec{c}_2)$, then it holds $\T\cup\C\not\modelseql\P$ for at least one CQ-censor $\C$ for $\E$. This, however, contradicts Definition~\ref{def:cq-censor}, thus proving that $\T\cup\A'\modelseql\P$.
Then, it is immediate to verify that, given any $q\in\BCQk$, $\E\modelscqe q$ iff $\E'\modelscqe q$.
\qed


\bigskip
\noindent
\textbf{Proof of Proposition~\ref{pro:indistinguishability-icqe}}

Let $\A'$ be the ABox isomorphic to the finite set of BCQs $\Q=\{ q\in\BCQh \mid \E\modelscqe q \}$, where $h=\max(k,\maxlengthatoms(\P))$. 
First, as shown in the proof of Proposition~\ref{pro:indistinguishability-cqe}, we have that $\T\cup\A'\modelseql\P$.
Then, since $\T\cup\A'\modelseql\P$, it is immediate to verify that, for every $q\in\BUCQk$, $\E\modelsicqe q$ iff $\E'\modelsicqe q$ (the key property is that a BUCQ is \icqe-entailed iff at least one of the BCQs occurring in it is \icqe-entailed).
\qed


\bigskip
\noindent
\textbf{Proof of Theorem~\ref{thm:indistinguishability-icqe}}

This property is an immediate consequence of Proposition~\ref{pro:indistinguishability-icqe} when we assume that $k$ is the maximum length of a BCQ occurring in a BUCQ of $\Q$.
\qed


\bigskip
\noindent
\textbf{Proof of Lemma~\ref{lem:cqe-entailment-bcq-lb}}

We show a reduction of 3-CNF, a well-known NP-hard problem, to the complement of \cqe-\decprob[\BCQ].
Let $\T$ be the empty TBox, and let $\P$ contain the following EDs:
\[
\begin{array}{r@{}l}
\forall x,y,v,z\,
\big( K (&C_1(x,y)\wedge V_1(x,v)\wedge V(y,v) \wedge \\&\quad N(x,z)\wedge S(x) ) \rightarrow K S(z) \big) \\
\forall x,y,v,z\,
\big( K (&C_2(x,y)\wedge V_2(x,v)\wedge V(y,v) \wedge \\&\quad N(x,z)\wedge S(x) ) \rightarrow K S(z) \big) \\
\forall x,y,v,z\,
\big( K (&C_3(x,y)\wedge V_3(x,v)\wedge V(y,v) \wedge \\&\quad N(x,z)\wedge S(x) ) \rightarrow K S(z) \big) \\
\forall x \, \big(K (&V(x,f)\wedge V(x,t)) \rightarrow K \bot \big)
\end{array}
\]

Given a 3-CNF formula $\phi$ with $m$ clauses, we represent every clause of $\phi$ in the ABox $\A$ through the roles $C_1,C_2,C_3,V_1,V_2,V_3$: e.g., if the $i$-th clause of $\phi$ is $\neg a \vee b \vee \neg c$, we add to $\A$ the assertions $C_1(i,a),C_2(i,b),C_3(i,c),V_1(i,f),V_2(i,t),V_3(i,f)$. Moreover, $\A$ contains the assertions
\[ 
\begin{array}{l}
\{ V(a,f),V(a,t) \mid a\in\varprop(\phi)\}\: \cup \\
\{ S(i), N(i,i+1) \mid 1\leq i\leq m \} 
\end{array}
\]
where $\varprop(\phi)$ are the propositional variables of $\phi$.

We now show that $\phi$ is satisfiable iff $\tup{\emptyset,\A,\P}\not\modelscqe S(1)$.

First, if $\phi$ is satisfiable, then let $P$ be the set of propositional variables occurring in $\phi$, let $I$ be a propositional interpretation (subset of $P$) satisfying $\phi$, and let $\C$ be the following subset of $\bcqclosure(\A)$:
\[
\begin{array}{r@{}l}
    \C = \A \setminus (
    &\{V(p,f) \mid p\in P\cap I \}\;\cup\\
    &\{ V(p,t) \mid p\in P\setminus I \} \cup \{S(1),\ldots,S(m)\} )
\end{array}
\]
    
It is immediate to verify that $\C\modelseql\P$, and thus there exists an optimal CQ-censor $\C'$ of $\E$ containing $\C$.
Now suppose that $S(1)\in\C'$. Then, since $I$ satisfies $\phi$, there exists an instantiation of $y$ and $v$ that, together with the instantiation $x=1$, $z=2$ is such that the BCQ in the premise of one of the first three dependencies of $\P$ is entailed by $\C'$, which implies (for $\C'\modelseql\P$ to hold) that $S(2)\in\C'$. But now, the presence of $S(2)$ in $\C'$ implies (for the same argument as above) that $S(3)\in\C'$, Iterating the above argument, we conclude that $S(m)\in\C'$: this in turn would imply that $S(m+1)\in\C'$, but $S(m+1)$ does not belong to $\bcqclosure(\A)$, and since $\C'\subseteq\bcqclosure(\A)$, we have a contradiction.
Consequently, $\C'$ can not contain $S(1)$, which implies that $\E\not\modelscqe S(1)$.

On the other hand, given a guess of the atoms of the $V$ predicate in $\A$ satisfying the fourth dependency and corresponding to an interpretation of the propositional variables that does not satisfy $\phi$, it is immediate to verify that the sequence of instantiations of the bodies of the first three dependencies of $\Sigma$ mentioned above (which, as above explained, leads to the need of adding $S(m+1)$ to the optimal CQ-censor when the interpretation corresponding to such a guess of the atoms for $V$ satisfies $\phi$) is blocked by the absence of some ground atom for $V$ in the optimal CQ-censor. More precisely: there exists a positive integer $k\geq 1$ such that the atoms $S(k), S(k-1), \ldots, S(1)$ can be added to all the optimal CQ-censors corresponding to such a guess of the $V$ atoms. This implies that, if $\phi$ is unsatisfiable, then $S(1)$ belongs to all the optimal CQ-censors of $\E$, hence $\E\modelscqe S(1)$.
\qed


\bigskip
\noindent
\textbf{Proof of Lemma~\ref{lem:cqe-entails}}

Suppose the algorithm returns false. Then, there exists a set of BCQ$_k$s $\C$ satisfying conditions (i), (ii), and (iii) of the algorithm. Now, by condition (i) it follows that there exists an optimal CQ-censor $\C'$ of $\E$ such that $\C\subseteq\C'$. Then, since condition (iii) of the algorithm holds, by Lemma~\ref{lem:kcons} it follows that $\C=\C'\cap\BCQk$ (hence $\C$ is closed under subqueries); consequently, by condition (ii) and Proposition~\ref{pro:dlliter-bcqk-projection} it follows that $\T\cup\C'\not\models q$, which implies that $\E\not\modelscqe q$. 

Conversely, suppose $\E\not\modelscqe q$. then, there exists an optimal CQ-censor $\C'$ of $\E$ such that $\T\cup\C'\not\models q$. Now let $k=\max(\maxlength(q),\maxlengthatoms(\P))$ and let $\C=\C'\cap\BCQk$. Since $\C$ is closed under subqueries, by Proposition~\ref{pro:dlliter-bcqk-projection}, for every subformula $Kq(\vec{x})$ occurring in $\P$ and for every ground substitution $\vec{t}$ for $\vec{x}$, where $\vec{t}$ is a tuple of constants occurring in $\A$, $\T\cup\C\models q(\vec{t})$ iff $\T\cup\C'\models q(\vec{t})$, hence by Proposition~\ref{pro:eqllite-ucq-eval} it follows that $\T\cup\C\modelseql\P$, therefore condition (i) of the algorithm holds for $\C$.
Moreover, by Proposition~\ref{pro:dlliter-bcqk-projection} it follows that condition (ii) holds for $\C$.
Finally, suppose condition (iii) does not hold for $\C$, i.e.\ there exists $q'\in \bcqkclosure(\T\cup\A) \setminus \C$ such that $\kcons(\T,\P,\C\cup\{q'\})\subseteq\bcqkclosure(\T\cup\A)$. 
Then, it is immediate to verify that $\T\cup\kcons(\T,\P,\C\cup\{q'\})\modelseql\P$, i.e.\ $\kcons(\T,\P,\C\cup\{q'\})$ is a CQ-censor of $\E$. 
Now let $\C''$ be the maximal subset of $\bcqclosure(\T\cup\A)$ such that $\kcons(\T,\P,\C\cup\{q'\})=\C''\cap\BCQk$ and $\C''=\bcqclosure(\T\cup\C'')$. By Proposition~\ref{pro:dlliter-bcqk-projection} and Proposition~\ref{pro:eqllite-ucq-eval} it follows that $\T\cup\C''\modelseql\P$. But now, $\C'$ is the maximal subset of $\bcqclosure(\T\cup\A)$ such that $\C=\C'\cap\BCQk$, and since $\C\subset\kcons(\T,\P,\C\cup\{q'\})$, it follows that $\C'\subset\C''$, thus contradicting the hypothesis that $\C'$ is an optimal CQ-censor of $\E$. Consequently, condition (iii) holds for $\C$, hence the algorithm returns false.
\qed


\bigskip
\noindent
\textbf{Proof of Lemma~\ref{lem:rewriting-consistency}}

Suppose $\eval(q_r(\vec{c}),\I_\A)$ is true. Then, $\A\models q_r(\vec{c})$, hence by Proposition~\ref{pro:rewriting-dlliter} it follows that $\T\cup\A\models q_r'(\vec{c})$, where $q_r'(\vec{x})=\phikcons(\T,\P,q(\vec{x}))$, which implies that $\kcons(\T,\P,\{q(\vec{c})\})\subseteq\bcqclosure(\T\cup\A)$, thus by Lemma~\ref{lem:kcons} there exists an optimal CQ-censor of $\E$ that contains $q(\vec{c})$.

Conversely, suppose there exists an optimal CQ-censor of $\E$ that contains $q(\vec{c})$. Then, by Lemma~\ref{lem:kcons} $\kcons(\T,\P,\{q(\vec{c})\})\subseteq\bcqclosure(\T\cup\A)$. Thus, $\T\cup\A\models q'$ for every BCQ $q'\in \kcons(\T,\P,\{q(\vec{c})\})$, which implies that $\T\cup\A\models q_r'(\vec{c})$. Hence, by Proposition~\ref{pro:rewriting-dlliter} we derive $\A\models \perfectref(q_r'(\vec{c}),\T)$.
Now, from the properties of $\perfectref$ it immediately follows that the BCQ $\perfectref(q_r'(\vec{c}),\T)$ is equivalent to $q_r(\vec{c})$, hence $\A\models q_r(\vec{c})$, and therefore $\eval(q_r(\vec{c}),\I_\A)$ is true.
\qed


\bigskip
\noindent
\textbf{Proof of Lemma~\ref{lem:clash}}

Suppose $q'(\vec{c})$ is a clash for $q$ in $\E$, Then, there exists an optimal CQ-censor of $\E$ containing $q'(\vec{c})$, which by Lemma~\ref{lem:rewriting-consistency} implies that $\eval(q_r(\vec{c}),\I_\A)$ is true, where $q_r(\vec{x})=\perfectref(\phikcons(\T, \P, q'(\vec{x})),\T)$. Moreover, for every $q_i\in q$, $q'(\vec{c})$ is a clash for $q_i$ in $\E$, i.e.\ there exists no optimal CQ-censor of $\E$ containing $q'(\vec{c})\wedge q_i$, which again by Lemma~\ref{lem:rewriting-consistency} implies that $\eval(\bigwedge_{q_i\in q}\neg q_r'(\vec{c}),\I_\A)$ is true, where $q_r'(\vec{x})=\perfectref(\phikcons(\T, \P, q'(\vec{x})\wedge q_i),\T)$. Consequently, $\eval(q_{cl}(\vec{c}),\I_\A)$ is true.

Conversely, suppose $\eval(q_{cl}(\vec{c}),\I_\A)$ is true. This implies that $\eval(q_r(\vec{c}),\I_\A)$ is true, hence by Lemma~\ref{lem:rewriting-consistency} it follows that there exists an optimal CQ-censor of $\E$ containing $q'(\vec{c})$. Moreover, since $\eval(q_{cl}(\vec{c}),\I_\A)$ is true, then for every $q_i\in q$, $\eval(\neg q_r'(\vec{c}),\I_\A)$ is true, hence $\eval(q_r'(\vec{c}),\I_\A)$ is false, and thus by Lemma~\ref{lem:rewriting-consistency} it follows that there exists no optimal CQ-censor of $\E$ containing $q'(\vec{c})\wedge q_i$. Consequently, $q'(\vec{c})$ is a clash for $q$ in $\E$.
\qed



\bigskip
\noindent
\textbf{Proof of Lemma~\ref{lem:clash-dlliter-acyclic}}

To prove the lemma, we need to introduce some auxiliary definitions and properties.

Given an ED $\delta$ of the form $\forall\vec{x}\,(K q_b\rightarrow K q_h)$, we denote by $\body(\delta)$ the CQ $q_b$ and denote by $\head(\delta)$ the CQ $q_h$.
Also, we use the notation $\delta(\vec{x})$ to indicate an ED with universally quantified variables $\vec{x}$.

\newcommand{\edclosure}{\textit{ED-closure}}

Given an ED $\delta$ of the form $\forall x_1,\ldots,x_m\, (K q_b \rightarrow K q_h)$,
we define $\TGD(\delta)$ as the following FO implication corresponding to a \emph{tuple-generating dependency} (\emph{TGD}), also known as \emph{existential rule} (see e.g.\ \cite{CaliGL12,BLMS11}):
\[
\forall x_1,\ldots,x_m\, (q_b \wedge\bigwedge_{i=1}^m\constpred(x_i) \rightarrow q_h)
\]
Moreover, given a policy $\P$, we define $\TGD(\P)$ as the set of TGDs $=\{\TGD(\delta)\mid\delta\in\P\}$.

Then, given a $\dlliter$ positive TBox inclusion $\eta$, we define $\TGD(\eta)$ as the TGD corresponding to $\eta$, i.e.:
\[
\begin{array}{l@{\ =\ }l}
  \TGD(A\ISA B) & \forall x (A(x)\rightarrow B(x)) \\
  \TGD(A\ISA \exists s) & \forall x (A(x)\rightarrow \exists z\, s(x,z)) \\
  \TGD(A\ISA \exists s^-) & \forall x (A(x)\rightarrow \exists z\, s(z,x)) \\
  \TGD(\exists r\ISA A) & \forall x,y (r(x,y)\rightarrow A(x)) \\
  \TGD(\exists r\ISA \exists s) & \forall x,y (r(x,y)\rightarrow \exists z \, s(x,z)) \\
  \TGD(\exists r\ISA \exists s^-) & \forall x,y (r(x,y)\rightarrow \exists z \, s(z,x)) \\
  \TGD(\exists r^-\ISA A) & \forall x,y (r(y,x)\rightarrow A(x)) \\
  \TGD(\exists r^-\ISA \exists s) & \forall x,y (r(y,x)\rightarrow \exists z \, s(x,z)) \\
  \TGD(\exists r^-\ISA \exists s^-) & \forall x,y (r(y,x)\rightarrow \exists z \, s(z,x)) \\
  \TGD(r\ISA s) & \forall x,y (r(x,y)\rightarrow s(x,y)) \\
  \TGD(r\ISA s^-) & \forall x,y (r(x,y)\rightarrow s(y,x)) \\
  \TGD(r^-\ISA s) & \forall x,y (r(y,x)\rightarrow s(x,y)) \\
  \TGD(r^-\ISA s^-) & \forall x,y (r(y,x)\rightarrow s(y,x))  
\end{array}  
\]
We also define $\TGD(\T)=\{\TGD(\eta) \mid \eta\in\T\}$ for a $\dlliter$ TBox $\T$.
Finally, we denote by $\TGD(\P,\T)$ the set $\TGD(\P)\cup\TGD(\T)$.

\ 

\newcommand{\noconstpred}{\textit{No}\constpred}

In the following, given a CQ $q(\vec{x})$, we denote by $\noconstpred(q(\vec{x}))$ the CQ obtained from the $q(\vec{x})$ by eliminating the atoms with predicate $\constpred$.

We now make use of an existing query rewriting algorithm for TGDs, in particular the algorithm shown in~\cite{DBLP:journals/semweb/KonigLMT15}
, which computes a UCQ rewriting (i.e.\ a set of CQs) of a CQ with respect to a set of TGDs. More precisely, we denote by $\ucqrewrite(q(\vec{x}),\Sigma)$ the set of CQs returned by such an algorithm for the input CQ $q(\vec{x})$ and set of TGDs $\Sigma$.

We first show the following property:
\begin{quote}
    (PR1): Let $\delta(\vec{x})\in\P$. Then, the algorithm $\ucqrewrite$ on input $\body(\delta(\vec{x}))$, $\TGD(\P,\T)$ terminates and, for every $q'(\vec{x}')\in\ucqrewrite(\body(\delta(\vec{x})),\TGD(\P,\T))$, it holds $\cqlength(\noconstpred(q'(\vec{x}')))\leq k^h$.
\end{quote}

Proof of (PR1): It is not hard to verify that, since $\P$ is an acyclic policy and $\T$ is a $\dlliter$ TBox, the algorithm $\ucqrewrite(q,\TGD(\P,\T))$
terminates on every CQ $q$ (and thus in particular on $\body(\delta(\vec{x}))$.
Indeed, every atom whose predicate is $\constpred$ can not be rewritten at all (since the predicate $\constpred$ does not appear in the right-hand side of any TGD in $\TGD(\P,\T)$), while every atom whose predicate is not $\constpred$ can be rewritten at most once by TGDs from $\TGD(\P)$ (because $\P$ is acyclic for $\T$). Moreover, the set $\TGD(\T)$ is linear, which implies that only a finite number of consecutive applications of such TGDs can be applied to the rewriting of any atom.
Moreover, since the application of the TGDs in $\TGD(\T)$ in the rewriting cannot increase the size of the rewritten query, it follows that every atom of $\body(\delta(\vec{x}))$ can be rewritten in no more than $k^{h-1}$ atoms: since the length of $\body(\delta(\vec{x}))$ is not greater than $k$, it follows that every CQ returned by $\ucqrewrite(\body(\delta(\vec{x})),\TGD(\P,\T))$ has a length not greater than $k^h$, thus property (PR1) holds.

\

We now make use of $\ucqrewrite$ to define $\edclosure(\P,\T)$ as the following set of EDs:

\begin{small}
\[
\begin{array}{l}
\{ \forall \vec{x}' (K \noconstpred(q'(\vec{x}'))\rightarrow K \head(\delta(\vec{x}')) \mid \\
\;\delta(\vec{x})\in\P \wedge q'(\vec{x}')\in\ucqrewrite(\body(\delta(\vec{x})),\TGD(\P,\T)) \}
\end{array}
\]
\end{small}

\

Then, we prove the following property:
\begin{quote}
(PR2): Let $q$ be a BCQ such that $\T\cup\A\models q$. There exists an optimal CQ-censor of $\E$ containing $q$ iff, for every ED $\delta(\vec{x})\in\edclosure(\P,\T)$,
and for every ground substitution $\vec{c}$ of $\vec{x}$, if $q\models\body(\delta(\vec{c}))$, then $\T\cup\A\models\head(\delta(\vec{c}))$.
\end{quote}

Proof of (PR2):
Suppose that there exists an optimal CQ-censor $\C$ of $\E$ containing $q$. Now suppose there exists $\delta(\vec{x})\in\edclosure(\P,\T)$ and a ground substitution $\vec{c}$ of $\vec{x}$ such that $\{q\}\models\body(\delta(\vec{c}))$ and $\T\cup\A\not\models\head(\delta(\vec{c}))$. Of course, since $q\in\C$, it follows that $\T\cup\C\models\body(\delta(\vec{c}))$. Now let $\delta_0(\vec{x})$ be the ED in $\P$ that has been rewritten into $\delta(\vec{x})$. Given the properties of $\ucqrewrite$, it follows that, if $\T\cup\C\models\body(\delta(\vec{c}))$, then $\T\cup\C\models\body(\delta_0(\vec{c}))$. Consequently, since $\C$ is an optimal CQ-censor of $\E$, $\T\cup\C\models\head(\delta(\vec{c}))$, and since $\C\subseteq\bcqclosure(\T\cup\A)$, it follows that $\T\cup\A\models\head(\delta(\vec{c}))$.

Conversely, suppose there exists no optimal CQ-censor of $\E$ containing $q$. Then, by Lemma~\ref{lem:kcons}, $\kcons(\T,\P,\{q\})\not\subseteq\bcqclosure(\T\cup\A)$.
By Definition~\ref{def:kcons}, it follows that there exist $\delta(\vec{x})\in\P$ and a ground substitution $\vec{c}$ for $\vec{x}$ such that $\T\cup\kcons(\T,\P,\{q\})\models\body(\delta(\vec{c}))$ and $\T\cup\A\not\models\head(\delta(\vec{c}))$.

Now, it is easy to verify that $\T\cup\kcons(\T,\P,\{q\})\models \body(\delta(\vec{c}))$ iff $\chase(q,\TGD(\P,\T))\models\body(\delta(\vec{c}))$, where $\chase$ is a well-known algorithm \cite{AbHV95} that, given a set of atoms\footnote{As often done in previous work, here we apply the chase to a BCQ rather than to a set of ground atoms: it suffices to consider the BCQ as the set of its atoms, and its variables as a new and additional set of constants.} and a set of TGDs, returns a (possibly infinite) set of atoms that contains the initial set and is isomorphic to a \emph{canonical model} (or minimal model) of all the input formulas.

Moreover, the known properties of $\ucqrewrite$ imply that
$\chase(q,\TGD(\P,\T))\models \body(\delta(\vec{c}))$ iff there exists $q'(\vec{x}')\in\ucqrewrite(\body(\delta(\vec{x}),\TGD(\P,\T)))$ such that $\{q\}\models q'(\vec{c})$.

Consequently, by the definition of $\edclosure$, there exists $\delta'(\vec{x}')\in\edclosure(\P,\T)$ such that $\{q\}\models\body(\delta'(\vec{c}))$ and $\head(\delta'(\vec{c}))=\head(\delta(\vec{c}))$, and thus $\T\cup\A\not\models\head(\delta'(\vec{c}))$, which concludes the proof.










As a consequence of (PR2), we immediately obtain the following property:
\begin{quote}
(PR3): $q'$ is a clash for $q$ in $\E$ iff, for every $q_i\in q$, there exist an ED $\delta(\vec{x})\in\edclosure(\P,\T)$ and a ground substitution $\vec{c}$ for $\vec{x}$ such that $\{q'\wedge q_i\}\models\body(\delta(\vec{c}))$.
\end{quote}

Finally, it is immediate to verify that $\E\modelscqe q$ iff there exists no BCQ $q'$ such that $q'$ is a clash for $q$ in $\E$. This property, property (PR1), and property (PR3) immediately imply the lemma.
\qed


\bigskip
\noindent
\textbf{Proof of Lemma~\ref{lem:rewriting-cqe-entailment}}

First, let $\psi_1(q^p)$ be the sentence
\[
\begin{array}{l}
\displaystyle
 \bigwedge_{q_i\in q^p} \perfectref(q_i,\T)\: \wedge 
 \bigwedge_{q_i\in q\setminus q^p}\neg\perfectref(q_i,\T)
\end{array}
\]
and let $\psi_2(q^p)$ be the sentence 
\[
\begin{array}{l}
\displaystyle
\bigwedge_{q'(\vec{x})\in \Q}  
\neg \big( \exists \vec{x} \,(\clash(q^p,q'(\vec{x}),\T,\P) \wedge \bigwedge_{x\in\vec{x}} Ind(x) )
\end{array}
\]

Now, suppose $\eval(\censentailed(q,\T,\P),\I_\A)$ is true. Thus,
there exists a non-empty subset $q^p$ of $q$ such that $\eval(\psi_1(q^p)\wedge\psi_2(q^p),\I_\A)$ is true. Then, since $\eval(\psi_1(q^p),\I_\A)$ is true, it follows that $q_i\in\bcqclosure(\T\cup\A)$ for every $q_i\in q$. Therefore, 
from Lemma~\ref{lem:clash} 
and from the fact that $\eval(\psi_2(q^p),\I_\A)$ is true, it follows that there exists no BCQ $q'$ such that 
$q'\in\Q'$
and $q'$ is a clash for $q^p$ in $\E$,
where $\Q'$ is the set of BCQs defined over the predicates and constants occurring in $\{q\} \cup\P \cup \T$ and whose maximum length is $m\cdot k^h$, with
$m$ representing the number of BCQs in $q$, $h$ representing the number of EDs in $\P$, and $k=\maxlengthatoms(\P)$.
Consequently, by Lemma~\ref{lem:clash-dlliter-acyclic} we conclude that $\E\modelscqe q^p$, which immediately implies that $\E\modelscqe q$.

Conversely, suppose $\E\modelscqe q$. Let $q^p$ be the non-empty subset of BCQs $\{q_i\in q \mid q_i\in\bcqclosure(\T\cup\A)\}$. Then, $\eval(\psi_1(q^p),\I_\A)$ is true. Moreover, since there exists no BCQ $q'$ such that $q'\in\BCQ$ and $q'$ is a clash for $q$ in $\E$, by Lemma~\ref{lem:clash} it follows that $\eval(\psi_2(q^p),\I_\A)$ is true. Consequently, $\eval(\censentailed(q,\T,\P),\I_\A)$ is true.
\qed